\documentclass[letterpaper, 10 pt, conference, shadows]{ieeeconf}  

\IEEEoverridecommandlockouts                              

\pdfoutput=1

\usepackage{longtable,tabularx}
\usepackage{tcolorbox}
\usepackage{xcolor,colortbl}
\usepackage{diagbox}
\usepackage{multirow}

\usepackage{graphicx} 
\usepackage[font=small]{subcaption}
\usepackage{epsfig} 
\usepackage{amsmath} 
\usepackage{amssymb}  
\usepackage{color}
\usepackage{cite}
\usepackage{float}
\usepackage{booktabs}
\usepackage{eso-pic}

\usepackage[ruled, linesnumbered, noend]{algorithm2e}

\usepackage{adjustbox}
\usepackage{tikz}
\usetikzlibrary{automata,positioning,shapes,arrows,shadows,shadows.blur}
\newcommand{\shadowString}{circular drop shadow}

\makeatletter
\newcommand{\removelatexerror}{\let\@latex@error\@gobble}
\makeatother

\usepackage{amsthm}
\newtheorem{definition}{Definition}
\newtheorem{problem}{Problem}
\newtheorem{example}{Example}
\newtheorem{remark}{Remark}
\newtheorem{lemma}{Lemma}

\SetKwComment{Comment}{// }{}
\SetKw{Continue}{continue}

\renewcommand{\phi}{\varphi}

\definecolor{Gray}{gray}{0.85}
\definecolor{LightCyan}{rgb}{0.88,1,1}

\newcolumntype{a}{>{\columncolor{Gray}}c}

\newcommand{\traj}{\tau}
\newcommand{\cost}[1]{c(#1)}
\newcommand{\Cost}[1]{C(#1)}

\newcommand{\pcs}{\mathcal{C}}
\newcommand{\pcst}[1]{\pcs(#1)}

\newcommand{\acc}[1]{\mathrm{acc}(#1)}
\newcommand{\trace}{\sigma}
\newcommand{\plan}{\pi}
\newcommand{\csltl}{scLTL }
\newcommand{\exstate}[2]{x_{#1, #2}}
\newcommand{\paut}{\mathcal{P}}
\newcommand{\ptraj}{\rho}
\newcommand{\pacc}{p_{acc}}
\newcommand{\dfa}{\mathcal{A}}

\title{\LARGE \bf
Optimal Cost-Preference Trade-off Planning with Multiple \\Temporal Tasks
}

\author{{Peter Amorese and Morteza Lahijanian}
\thanks{This work was supported in part by University of Colorado Boulder and NASA COLDTech Program under grant \#80NSSC21K1031.}%
\thanks{Authors are with the department of Aerospace Engineering Sciences at the University of Colorado Boulder, CO, USA
        {\tt\small \{\textit{firstname}.\textit{lastname}\}@colorado.edu}}%
}

\begin{document}
\AddToShipoutPictureBG*{%
  \AtPageUpperLeft{%
    \hspace{16.5cm}%
    \raisebox{-1.5cm}{%
      \makebox[0pt][r]{To appear in the Int'l. Conference on Intelligent Robots and Systems (IROS), 2023.}}}}
\maketitle

\thispagestyle{plain}
\pagestyle{plain}

\begin{abstract}
Autonomous robots are increasingly utilized in realistic scenarios with multiple complex tasks. In these scenarios, there may be a preferred way of completing all of the given tasks, but it is often in conflict with optimal execution. 
Recent work studies preference-based planning, however, they have yet to extend the notion of preference to the behavior of the robot with respect to each task. 
In this work, we introduce a novel notion of preference that provides a generalized framework to express preferences over individual tasks as well as their relations.
Then, we perform an optimal trade-off (Pareto) analysis between behaviors that adhere to the user's preference and the ones that are resource optimal.
We introduce an efficient planning framework that generates Pareto-optimal plans given user's preference by extending $A^*$ search. Further, we show a method of computing the entire Pareto front (the set of all optimal trade-offs) via an adaptation of a multi-objective $A^*$ algorithm.  We also present a problem-agnostic search heuristic to enable scalability.
We illustrate the power of the framework on both mobile robots and manipulators. Our benchmarks show the effectiveness of the heuristic with up to 2-orders of magnitude speedup.
\end{abstract}

\section{Introduction}
\label{sec: intro}




Recent advancements have enabled robots to take on dominant roles in our daily lives. Examples include self-driving cars, home assistive robots, and cooperative manipulators in factories. 
As they become more integrated into society, there is an increasing demand for them to perform not only multiple simultaneously-assigned complex tasks efficiently (e.g., minimizing time and energy) but also to complete them in a way that adheres to human preferences (e.g., some should be completed within a time window and others in a particular order). These objectives, however, are often competing;  that is, by optimizing for efficiency, the execution becomes less preferred.  This results in a trade-off between the objectives, and hence, the goal becomes to optimize for this trade-off.  Nevertheless, there are often many optimal trade-offs, known as the \emph{Pareto front}. Each point on the front corresponds to a unique behavior; hence, to pick a favorable behavior, it is usually desirable to know the whole front.
This poses several challenges from the planning perspective:  
(i) the notion of \emph{preference} is not trivial, (ii) generating the whole Pareto front is computationally intensive, and (iii) the planning space is immensely large due to the multiplicity and complexity of tasks. 
This paper aims to address these challenges by developing an efficient computational framework for planning for optimal trade-offs between \textit{cost} and \textit{preference} for multiple complex tasks.


Consider a robotic manipulator working in a restaurant as a cook. The robot receives several food orders. Firstly, the robot should cook all dishes in a minimal time. Secondly, the customers expect that the dishes are served according to the sequence in which each customer placed the order. It is likely that preparing and cooking ingredients simultaneously would be very efficient, but may not respect the order in which the customers should receive their dish. Thus, the robot must determine the appropriate optimal trade-offs between minimizing cooking time and adhering to the preferred order. 

A popular approach to express complex tasks in robotics is the use of temporal logics, 
namely \textit{Linear Temporal Logic} (LTL) \cite{BaierBook2008} for its expressive power \cite{Lahijanian:AR-CRAS:2018}. Planning with LTL 
specifications has been widely explored for manipulation applications \cite{He:ICRA:2015,He:IROS:2017,Muvvala:2022:ICRA}, and mobile 
robotic applications \cite{Kress:TRO:2009,Lahijanian:AR-CRAS:2018,Lahijanian:ICRA:2009}. 
Traditionally, a major limitation of those methods is the inability to scale to large, more realistic applications.  In addition,
because these applications aim to provide formal guarantees, they fail to reason about preferences or complex soft-constraints when planning with multiple tasks.

Recent work has introduced several notions of 
\textit{preference} in planning applications \cite{jorge2008planning}. A \textit{preference} can generally be described through quantitative (cost-based) or qualitative (comparison-based) measures. Qualitative preferences can be expressed through structures such as a Hierarchical Task Networks \cite{sohrabi2008planning,sohrabi2009htn,sohrabi2010preference,georgievski2014overview}, where the user is responsible for qualitatively comparing specification formulae through a binary relation. 
These methods struggle to be able to plan using an \textit{incomplete} binary relation, i.e., there exist at least two outcomes that cannot be compared, while reasoning about temporal objectives. 
Work \cite{kulkarni2022opportunistic} tackles planning with incomplete preferences over temporal goals. The quality of the outcome may be 
dependent on how extensive the user's given binary relation is. Since it is cumbersome for a user to provide an exhaustive binary relation for many outcomes, an alternative is to quantify the user's preference. 

Work \cite{bienvenu2006planning} introduces a language to compare qualitative preferences over temporal objectives where the user qualitatively annotates the level of preference for each formula. Similarly, the quantitative preference planning problem can be posed as a partial satisfaction \cite{lahijanian2015time, rahmani2019optimal} or minimum violation/maximum realizability \cite{vasile2017minimum, kamale2021automata, benton2012temporal, mehdipour2020specifying} planning problem. 
These approaches generally require the user to define costs for certain temporal outcomes based on the tasks alone. Many approaches look to Planning Domain Definition Language (PDDL) 
\cite{aeronautiques1998pddl} to specify quantitative preferences \cite{baier2009heuristic, gerevini2009deterministic, seimetz2021learning}. While these methods provide a framework for reasoning over temporal preferences, they are limited to reasoning about Boolean temporal behaviors, and are unable to reason about the efficiency and performance of the robotic system.

This work aims to provide a new perspective on preference-based planning that allows the user to quantify a preference with respect to behavior of an actual robotic system with respect to each given task. 
For example, the user can express their preference over the cumulative cost of completing each task in relation to the other tasks. 
Hence, in our framework, we reason about two planning objectives: (i) a user-defined \textit{preference objective}, and (ii) a cost objective that models the efficiency of the system.
Under this novel preference formulation, an optimal trade-off analysis can be performed to illustrate the robot's behavior for varying levels of preference objective adherence. 
We first introduce an adaptation of the $A^*$ search algorithm to efficiently synthesize a single optimal plan given a constraint on one objective, depending on how flexible the user's preference is. To compute the full Pareto front, we  leverage an adaptation of a recent bi-objective $A^*$ algorithm $BOA^*$ (Bi-Objective $A^*$) \cite{ulloa2020simple}, based off of the multi-objective $A^*$ search algorithm $NAMOA^*$ (New Approach to Multi-Objective $A^*$) \cite{mandow2008multiobjective}.


The contributions of this paper are as follows. Firstly, we introduce a novel generalized notion of preference over the cost of satisfying each individual task. 
Second, 
we define a problem-agnostic heuristic to speed up the computation time considerably. 
Third, we adapt the $BOA^*$ algorithm for computation of Pareto optimal solutions.
Lastly, we demonstrate our novelties by using case studies on both mobile robots and manipulators, as well as the effectiveness of the heuristic using benchmarks with results of up to 2-orders of magnitude speedup. 
\section{Problem Formulation}
\label{sec: problem_formulation}





\subsection{Robot Model}
We consider an abstracted model for a robot given as a transition system. Such abstractions are commonly used and constructed in formal approaches to both mobile robots \cite{Lahijanian:AR-CRAS:2018,Hadas:ICRA:2007,Lahijanian:ICRA:2009} and robotic manipulators \cite{He:ICRA:2015,He:RAL:2019}. 

\begin{definition}[WTS]
    A \emph{Weighted Transition System} (WTS) is a tuple $T = (S, s_0, A, \delta_T, c, AP, L)$ where
    \begin{itemize}
        \item $S$ is a finite set of states,
        \item $s_0 \in S$ is an initial state,
        \item $A$ is a finite set of actions,
        \item $\delta_T : S \times A \mapsto S$ is a transition function,
        \item $c : S \times A \mapsto \mathbb{R}_{\geq 0}$ is a cost function that assigns a non-negative weight to every state-action pair,
        \item $AP$ is a set of atomic propositions that is related to the robot task, and
        \item $L : S \mapsto 2^{AP}$ is a labeling function that maps each state to a subset of propositions in $AP$ that are true at that state.
    \end{itemize}
\end{definition}

A plan $\plan = \plan_0 \plan_1 \cdots \plan_{m-1}$ is a sequence of actions, where $m \geq 1$ and $\plan_k \in A$ for all $0\leq k \leq m-1$.
A trajectory $\traj = \traj_0 \traj_1 \cdots \traj_m$ is a finite sequence of states on WTS $T$, where $\traj_k \in S$ for all $0\leq k \leq m$.  
A plan is called \textit{valid} if it produces a trajectory $\traj^\plan = \traj_0^\plan \cdots \traj_m^\plan$  such that $\traj_0^\plan = s_0$ and $\traj_{k+1}^\plan = \delta_T(\traj_k^\plan, \plan_k)$ for all $0 \leq k \leq m-1$.
The \textit{observation} of $\traj^\plan$
is a trace $\trace^\plan = \trace_0^\plan \cdots \trace_{m}^\plan$ where $\trace_k^\plan = L(\traj_k^\plan) \subseteq AP$ for all $0\leq k \leq m$. 
Lastly, the \textit{total cost} of trajectory $\traj^\plan$, denoted by $\Cost{\traj}$, is the sum of all the state-action costs along $\traj^\plan$, i.e., 
\\
\vspace{-5mm}
\begin{equation}
    \label{eq:total cost}
    \Cost{\traj} = \sum_{k=0}^{m-1} \cost{\traj_k^\plan, \plan_k}.
\end{equation}

\begin{example}
    \emph{
    Consider a 2D mobile robot in a grid environment tasked with collecting resources as shown in Fig. \ref{fig: gridword_ex}. 
    Each state is a 2D coordinate $(x,y)$, denoted by $\exstate{x}{y}$, where the initial state $s_0 = \exstate{0}{0}$. There are 4 cardinal direction actions $A = \{North, South, East, West\}$. Function $\delta_T$ encodes transitions between neighboring cells based on the actions. Each transition is weighted by a cost objective $c$. 
    The labeling function $L$ marks the location of the resources, where $L(\exstate{2}{0})=\{dirt\}$, $L(\exstate{1}{1})=\{plant\}$, $L(\exstate{1}{2})=\{rock\}$, $L(\exstate{2}{2})=\{charge\}$, and $L(x) = \emptyset$ otherwise.
    }
    \label{ex:wts}
\end{example}
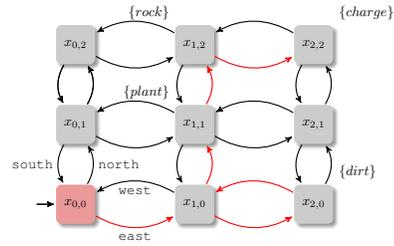
\begin{figure}[t!]
    \centering
    \begin{adjustbox}{scale=0.52}
    \centering
    \begin{tikzpicture}[shorten >=1pt,node distance=6cm,>=stealth',thick, auto,
                        every state/.style={fill,very thick,black!20,text=black,\shadowString},
                        robot/.style = {fill,very thick,black!20,rounded corners, text=black, shape=rectangle, minimum height=1cm,minimum width=1cm, blur shadow},
                        accepting/.style ={blue!50!black!50,text=white,accepting by double},
                        initial/.style ={red!80!black!40,text=black,initial by arrow, initial left}, initial text=$ $]
        \node[robot, initial] (s1) {$x_{0,0}$};
        \node[robot] (s2) [above=1cm of s1] {$x_{0,1}$};
        \node[robot] (s3) [above=1cm of s2] {$x_{0,2}$};
        \node[robot] (s4) [right=2cm of s1] {$x_{1,0}$};
        \node[robot, label=above left:\texttt{$\{plant\}$}] (s5) [above=1cm of s4] {$x_{1,1}$};
        \node[robot, label=above left:\texttt{$\{rock\}$}] (s6) [above=1cm of s5] {$x_{1,2}$};
        \node[robot,, label=above right:\texttt{$\{dirt\}$}] (s7) [right=2cm of s4] {$x_{2,0}$};
        \node[robot] (s8) [above=1cm of s7] {$x_{2,1}$};
        \node[robot, label=above right:\texttt{$\{charge\}$}] (s9) [above=1cm of s8] {$x_{2,2}$};

        \path[->]
           (s1) edge [bend right, below, red] node [black] {\texttt{east}} (s4)
           (s4) edge [bend right, above] node [below] {\texttt{west}} (s1)

           (s1) edge [bend right, below]  node [right] {\texttt{north}} (s2)
           (s2) edge [bend right, above]  node [left] {\texttt{south}}(s1)

           (s2) edge [bend right, below] (s3)
           (s3) edge [bend right, above] (s2)

           (s2) edge [bend right, below] (s5)
           (s5) edge [bend right, above] (s2)

           (s4) edge [bend right, below, red] (s7)
           (s7) edge [bend right, above, red] (s4)

           (s2) edge [bend right, below] (s3)
           (s3) edge [bend right, above] (s2)

           (s3) edge [bend right, below] (s6)
           (s6) edge [bend right, above] (s3)

           (s6) edge [bend right, below] (s5)
           (s5) edge [bend right, above, red] (s6)

           (s5) edge [bend right, below] (s4)
           (s4) edge [bend right, above, red] (s5)

           (s5) edge [bend right, below] (s8)
           (s8) edge [bend right, above] (s5)

           (s8) edge [bend right, below] (s7)
           (s7) edge [bend right, above] (s8)

           (s6) edge [bend right, below, red] (s9)
           (s9) edge [bend right, above] (s6)

           (s9) edge [bend right, below] (s8)
           (s8) edge [bend right, above] (s9);
    \end{tikzpicture}
    \end{adjustbox}
\caption{Weighted Transition System from Example~\ref{ex:wts}. The red arrows illustrate a satisfying plan for Example~\ref{ex:formulae}.}
\label{fig: gridword_ex}
\vspace{-3 mm}
\end{figure}


\subsection{Task Specification}
We are interested in a setting where the robot is given multiple tasks, each being a temporally extended goal that the robot must reach in finite time.
To specify such tasks, we use
\textit{Co-Safe Linear Temporal Logic} (\csltl) \cite{kupferman2001model}, which is a fragment of LTL \cite{baier2008principles}
that can be satisfied in finite time.
\begin{definition}[\csltl Syntax \cite{kupferman2001model}]
    A \emph{syntactically Co-Safe Linear Temporal Logic} (\csltl) formula over $AP$ is recursively defined as
    \begin{equation*}
        \phi = o \mid \neg o \mid \phi \land \phi \mid \phi \lor \phi \mid X\phi \mid \phi\, U \phi 
    \end{equation*}
    where $o \in AP$, $\neg$ (negation), $\land$ (conjunction), and $\lor$ (disjunction) are Boolean operators, and $X$ (next) and $U$ (until) are temporal operators.
    \label{def:csltl}
\end{definition}

\noindent
The commonly-used temporal operator ``eventually" ($F$) can be defined as $F\phi \equiv true \, U \phi$.
The semantics of \csltl are defined over infinite traces, but only a finite prefix of a trace is needed to satisfy a \csltl formula \cite{kupferman2001model}.
We denote the prefix of a trace $\trace$ up to time step $k$ by $\trace[k] = \trace_0 \ldots \trace_k$.

The robot is given $N > 1$ tasks $\phi_1, \ldots, \phi_N$. Ideally, we want the robot to achieve all the tasks, but it may be possible that some of the tasks are conflicting or  unachievable in an environment.  For ease of presentation, we assume the robot can complete all $N$ tasks, but we emphasize that our framework can be adapted for scenarios that tasks are not fully achievable using notions of \emph{partial satisfaction} introduced in \cite{Lahijanian:aaai:2015,Lahijanian:TRO:2016,tumova2013least,kim2013minimal}. We remark the required changes for this adaption throughout the paper.





\begin{definition}[Satisfying Plan]
    Given a set of \csltl formulae $\Phi = \{\phi_1, \ldots, \phi_N\}$, plan $\plan$ satisfies $\Phi$, denoted by $\pi \models \Phi$, 
    if it produces a valid trajectory $\traj^\pi$, whose observation trace $\trace^\plan$ satisfies all formulae in $\Phi$, i.e., 
    $\trace^\pi \models \wedge_{i=1}^N \phi_i$  
    Then, we say $\traj^\pi$ satisfies $\Phi$ denoted by $\traj^\plan \models \Phi$. 
    
\end{definition}
\noindent

\begin{remark}
    \label{remark: conflicting tasks}
    In the case that there exist tasks that are not achievable, instead of \emph{satisfying plan}, the notion of \emph{partially-satisfying plan} can be adopted from \cite{Lahijanian:aaai:2015}, which also quantifies \textit{distance to satisfaction} for each task.
\end{remark}

%

\begin{example}
    \emph{
    Following from Example \ref{ex:wts}, consider formulae:
    $\phi_1 = F\: charge$, \quad 
    $\phi_2 = F\:(plant \land F\:rock),$  \quad
    and $\phi_3 = \neg plant\:U\:dirt.$
    Translated to English, the formulae state that the robot must: ($\phi_1$) \emph{eventually $charge$}, ($\phi_2$) \emph{eventually collect $plant$, then collect $rock$ in that order}, and ($\phi_3$) \emph{do not collect $plant$ until $dirt$ has been collected.} An example of a satisfying plan is $\plan =$ \emph{East East West North North West}. 
    }
    \label{ex:formulae}
\end{example}

\subsection{Preferences}
\label{sec:preferences}
In this work, we are interested in plans that not only satisfy the tasks, but are also optimal with respect to two objectives: total cost in \eqref{eq:total cost} and user preferences over the cumulative action-cost of satisfying individual tasks. 


These objectives are possibly competing, i.e., by optimizing one, the other becomes sub-optimal. To formulate this problem, we first formalize the notions of individual cost and user preference.

Given task set $\Phi$ and a 
plan $\pi$, we are interested in the cost of trajectory $\traj^\plan$ relative to each task $\phi_i \in \Phi$. Let $K_i \geq 0$ be the length of the smallest prefix of $\trace^\plan$ that satisfies $\phi_i$, i.e.,
$\trace^\pi [K_i] \models \phi_i$ 
and $\trace^\pi [K_i-1] \not\models \phi_i$.  
Then, with an abuse of notation, we define the cost of $\traj^\pi$ with respect to $\phi_i$ 
to be \\
\vspace{-2mm}
\begin{equation}
    \label{eq:cost-phi}
    \Cost{\traj^\pi,\phi_i} = \sum_{k=0}^{K_i} c(\traj_k,\plan_k).
\end{equation}
\noindent



\begin{remark}
    For partial-satisfying plan, the cost of $\Cost{\traj^\pi,\phi_i}$ in \eqref{eq:cost-phi} can instead be derived from the notion of \textit{distance to satisfaction} for each task as in \cite{Lahijanian:aaai:2015}.
\end{remark}

We can now present a general notion of how users can articulate their preferences over \emph{any} trajectory by utilizing a single ordered set of individual task costs.



\begin{definition} [Preference Cost Set] \label{def:pcs}

    The \emph{preference cost set} (PCS) of task set $\Phi = \{\phi_1, \ldots, \phi_n\}$ by a trajectory $\traj$ is a vector of costs $\pcs(\traj,\Phi) = (c_1, \ldots, c_N)$, where
    $c_i = \Cost{\traj,\phi_i}$ if $\traj \models \phi_i$, otherwise, $c_i = \Cost{\traj}$.
    
    \label{def:sat_schedule}
\end{definition}


\noindent
The user defines their preference by constructing a function that assigns a non-negative value to $\pcs(\traj,\Phi)$. The lower the value of the function is, the more preferred $\traj$ is. 
\begin{definition}[Preference Function]
    Given a task set $\Phi$ and a valid trajectory $\traj$, a \emph{preference function} is a mapping $\mu : \mathbb{R}_{\geq 0}^N \mapsto \mathbb{R}_{\geq 0}$ that assigns a preference value to PCS $\pcs(\traj, \Phi)$
    according to how well it adheres to the user preference while satisfying two criteria:
    \begin{enumerate}
        \item $\mu$ is monotonically increasing along every valid trajectory, i.e., $\mu(\pcst{\traj_0 \ldots \traj_m, \Phi}) \leq \mu(\pcst{\traj_0 \ldots \traj_m \traj_{m+1}, \Phi})$ \;  $\forall$ $m \geq 0$, and
        \item given two trajectories $\traj_1$ and $\traj_2$, $\mu(\pcst{\traj_1, \Phi}) < \mu(\pcst{\traj_2, \Phi})$ iff $\pcst{\traj_1, \Phi}$ is preferred over $\pcst{\traj_2, \Phi}$. 
    \end{enumerate}
        
    \label{def:pref_func}
\end{definition}


This method of expressing preference is general to many applications of a user's preference. Since the user defines a function rather than values over PCS, it provides a flexible framework to capture many types of realistic preferences. 
For example, by simply combining each task cost (or the difference between task costs), a preference function can capture preferences such as a desired order of satisfaction among tasks, prioritizing the efficiency of select
tasks using priority-weights, etc. 

\begin{remark}
    By deriving PCS from partial-satisfaction costs, preferences such as incentives for certain tasks to remain below a violation threshold, prioritizing full-satisfaction of select tasks, etc. can be expressed.
\end{remark}


\begin{example}
    \emph{
    Imagine the user would like to quantify how ``out-of-order" the tasks are satisfied compared to the desired order of satisfaction, for example, the order of appearance in Example \ref{ex:formulae} ($\phi_1$ should be satisfied first, etc). 
    Consider a given input PCS $\pcs = (c_1, ... c_N)$, for example $\pcs = (20, 5, 10)$.
    Let the ``ideal" PCS $\pcs^* = (c_1^*, ... c_N^*)$ be equal to the sorted version of $\pcs$, e.g., $\pcs^* = (5, 10, 20)$. 
    A positive element $d_i$ in the resultant vector $\mathcal{D} = \pcs - \pcs^* = (15, -5, -10)$ indicates that $\phi_i$ was delayed (out of order) by $d_i$ units of cost. In this case, $\phi_1$ was delayed by $15$ units. Thus, a preference function $\mu$ can be defined as the sum of all positive elements of $\mathcal{D}$; in this case $\mu(\pcs)=15$.
    }
    
    \label{ex:pref_func}
\end{example}

\subsection{Cost and Preference Objectives}
    

We aim to find a plan that satisfies $\Phi$ and is optimal for both the total cost of the trajectory and the user preference.
The problem is formally stated as follows.

\begin{problem}
    \label{problem}
    Given a robot modeled as a WTS, a set of \csltl formulae $\Phi = \{\phi_1, \phi_2, \ldots \phi_N\}$, and a preference function $\mu$, 
    compute a plan $\pi$ for the robot such that the resulting trajectory $\traj^\plan$
    \begin{itemize}
        \item minimizes total cost $\Cost{\traj^\pi}$,
        \item minimizes preference function $\mu(\pcst{\traj^\pi, \Phi})$, and
        \item and satisfies all the tasks, i.e., $\pi \models \Phi$. 
    \end{itemize}
\end{problem}
\noindent

\begin{remark}
    If $\Phi$ is not fully achievable, the third objective in Problem~\ref{problem} only requires that $\plan$ \textit{partially} satisfy $\Phi$.
\end{remark}


Note that a plan that globally optimizes the cost and preference objectives may not exist because these objectives are often competing.
Thus, we seek to compute all the Pareto-optimal plans
that optimize the trade-off between cost and preference.
Further, note that Problem \ref{problem} is particularly challenging because the size of the search space grows exponentially with more complex formulae. Both single plan synthesis as well as the Pareto front computation may become intractable when solving large problems without a guiding heuristic.
\section{Approach}
\label{sec: approach}

We now present our approach to solving Problem~\ref{problem}. We begin by reducing Problem~\ref{problem} to graph-search.
Then, we augment $A^*$, a shortest-path graph search algorithm, to calculate a cost-optimal plan that adheres to a user-defined preference cost constraint. 
To deal with the natural intractability of Problem~\ref{problem}, we introduce a problem-agnostic heuristic that greatly improves the runtime of the algorithm. 

\subsection{Product Graph Construction}



From a \csltl formula $\phi_i$, a Deterministic Finite Automaton (DFA) $\dfa_i$ can be constructed that accepts precisely the same set of traces that satisfy $\phi_i$ \cite{kupferman2001model}.
\begin{definition}[DFA]
    A Deterministic Finite Automaton (DFA) is a tuple $\dfa = (Q, q_0, \Sigma, \delta_A, F)$ where: 
    $Q$ is a set of states,
    $q_0 \in Q$ is an initial state,
    $\Sigma \subset 2^{AP}$ is the alphabet,
    $\delta_A : S \times \Sigma \mapsto S$ is a transition function, and
    $F$ is a set of accepting (final) states.
\end{definition}

\noindent
Trace $\sigma \in (2^{AP})^*$ induces a run $\gamma = q_0 \dots q_m$, where $q_{k+1} = \delta_A(q_k,\sigma_k) \in Q$, on DFA $\dfa_i$.  Run $\gamma$ is accepting if $q_m \in F$.  Then trace $\sigma$ is accepted by $\dfa_i$ and hence satisfies $\phi_i$, i.e., $\sigma \models \phi_i$. 

A property of \csltl is accepting states $q_f\in F$ can be made absorbing, i.e, $\delta_A(q_f, \sigma) = q_f$ $\forall \sigma \in 2^{AP}$. 
The construction of DFA is worst-case doubly-exponential in the size of the task \cite{kupferman2001model}.

\begin{remark}
    In the partial satisfaction setting, a DFA for $\phi_i$ is augmented with extra transitions and transition weights, which represent allowable violations to task $\phi_i$ and the ``violation costs,'' respectively \cite{Lahijanian:aaai:2015}.  The obtained structure is a Weighted DFA (WDFA).
\end{remark}

Since we are interested in capturing progress towards each task simultaneously for a given trajectory on $WTS$ $T$, a product automaton $\paut$ can be defined, where each state combines the physical attributes from $T$ with temporal attributes from all $\phi_i \in \Phi$. 

\begin{definition}[Product Automaton] \label{def:paut}
  Given a WTS $T$ and DFAs $\dfa_1, \dfa_2, \ldots \dfa_N$, a Product Automaton $\paut = T \otimes \dfa_1 \otimes ... \otimes \dfa_N$ is
   a tuple $\paut= (P, p_0, A, \Sigma, \delta_P, C_P, acc)$, where 
   \begin{itemize}
       \item $P = S \times Q_1 \times ... \times Q_N$ is a set of states,
       \item $p_0 = (s_0, q_{1, 0}, q_{2, 0}, ..., q_{N, 0})$ is the initial state,
      \item $A$ is the same set of actions as in $T$,
      \item $\Sigma = 2^{AP}$ is the alphabet,
      \item $\delta_P : P \times A \mapsto P$ is a transition function where for any two states $p = (s, q_1, ..., q_N)$ and $p' = (s', q_1', ..., q_N')$, the transition $p' = \delta_P(p, a)$ exists if $s' = \delta_T(s, a)$ 
      and $q_i'= \delta_A(q_i, \mathcal{L}(s'))$ 
      for all $1\leq i\leq N$,
       \item $C_P : P \times A \times \{\Phi\} \mapsto \mathbb{R}_{\geq 0}^N$ is a function that weights each transition with a transition-PCS, and 
       \item $acc : P \mapsto 2^\Phi$ is an acceptance function that identifies which tasks are satisfied in each state, i.e. $\acc{(s, q_1, ..., q_N)} = \{\phi_i \mid q_i \in F_i \;\; \forall 1 \leq i \leq N \}$. 
   \end{itemize}
\end{definition}
\noindent
A plan $\pi$ generates a trajectory on $\paut$, denoted by $\rho = \ptraj_0 \ldots \ptraj_m$ where $\ptraj_0 = p_0$, $p_k\in P$ $\forall 0\leq k \leq m$, and 
$\ptraj_{k+1} = \delta_P(\ptraj_k, \plan_k)$.
To define $C_P$, we first define the transition-cost for each task as $c_P : P \times A \times \Phi \mapsto \mathbb{R}_{\geq 0}^N$ where, given a state $p = (s, q_0, q_1, ..., q_n)$ and action $a$, $c_P(p, a, \phi_i) = c(s, a)$. 
Since state $p\in P$ captures all of the relevant task history of any trajectory that reaches $p$, the construction of a PCS $\pcs(\ptraj, \Phi)$ can be alternatively expressed using the element-wise sum of  individual transition-PCS's. Given state-action pair $(p, a)$, $C_P$ constructs a transition-PCS defined as a tuple $C_P(p, a) = (\tilde{c}_{P,1}, \ldots, \tilde{c}_{P,N})$, where 

\begin{equation} \label{eq:transition-inheritance}
    \tilde{c}_{P,i} = 
    \begin{cases}
    c_P(p, a, \phi_i) &\text{if }  \phi_i \not\in \acc{p} \\
    0 &\text{otherwise}
    \end{cases}
\end{equation}
It follows form Def. \ref{def:pcs}, $\pcs(\ptraj,\Phi) = \sum_{k=0}^m C_P(\ptraj_k, a_k, \Phi)$. 

\begin{remark}
In the partial satisfaction setting, the product automaton is constructed from WDFAs instead of DFAs. Additionally, the transition cost for each task $c_P(p,a,\phi_i)$ is equal to the transition cost in the respective WDFA $\dfa_i$. 
\end{remark}


In traditional product construction, state $p = (s,q_0,\ldots,q_n)$ is labeled accepting only if $q_i \in F_i$ for all $1\leq i\leq n$, which means all the tasks are satisfied at $p$.
However, in our framework, as seen in \eqref{eq:transition-inheritance}, finer acceptance granularity is needed to correctly capture the PCS of a trajectory. Hence, we use an indicator function $\acc{p}$ that returns the subset of tasks that are satisfied at state $p$.
A satisfying trajectory $\ptraj_{acc}$ on $\paut$ is one that emits $\acc{\pacc} = \Phi$.

With this construction, determining the cost-optimal plan is reduced to graph search. Specifically, we desire a satisfying trajectory $\ptraj_{acc}$ that minimizes \eqref{eq:total cost}.
The cost-optimal satisfying plan $\plan^*$ can be determined using $\delta_P$ for the shortest path $\ptraj^*$ between node $p_0$ and $\pacc$ where $\acc{\pacc} = \Phi$. Note that since $\delta_P$ allows for multiple transitions between the same two states $p' = \delta_P(p, a)$ and $p' = \delta_P(p, a')$ under different actions $a \neq a'$, we choose the action with the smallest edge weight, i.e., 
action $a^*$ is selected if $c_P(p,a^*) \leq c_P(p, a)$.
\subsection{Constrained Symbolic Search}

\removelatexerror
\begin{algorithm}[t]
\caption{$A^*$ Constrained Synthesis}\label{alg:astar}
\KwIn{$\paut = (P, p_0, A, \Sigma, \delta_P, c_P, acc)$, $h$, $\mu$, $\mu_{max}$}
\KwOut{$\plan^*$}
Create a new node $x_0$ and add it to set $O$ \\
$(state(x_0), g(x_0), \pcs(x_0), f(x_0)) \gets (p_0, 0, zeros(N), 0)$ \\
Set $g_{min}$ to $0$ for $p_0$ and $\infty$ for all other states \\
$parent(x_0) \gets NULL$ \\
\While{$O$ is not empty} {
    Select and remove $x$ from $O$ such that $f(x) \leq f(x') \forall x'\in O$\\
    \If {$acc(state(x)) = \Phi$ and $g^+ \leq g_{min}(state(x))$} {
        \KwRet{$extractPlan(x)$} \\ \label{alg1:extractplan}
    }
    \For{$a\in actions(state(x))$} { \label{alg1:symb}
        $g^+ \gets g(x) + cost(p, a)$ \\
        $\pcs^+ \gets \pcs(x) + C_P(p, a)$ \\
        \uIf{$\mu(\pcs^+) > \mu_{max}$}{\label{alg1:prune2}
            \Continue
        } \uElseIf {$g^+ < g_{min}(\delta_P(p, a))$} {
            $g_{min}(\delta_P(p, a)) \gets g^+$ \\
        }
        Create a new node $y$ and add it to $O$ \\
        $(state(y), g(y), \pcs(y), f(y)) \gets (\delta_P(p, a), g^+, \pcs^+, g^+ + h(\delta_P(p, a))$ \\
        $parent(y) \gets x$
        
    }
}
\KwRet{failure}
\end{algorithm}

\label{subsec: symb_search}
For the case that the user does not want to perform the entire trade-off analysis, but rather they are interested in a single satisfying plan that adheres to a preference cost bound $\mu_{max}$, we provide a single-objective constrained synthesis algorithm that returns a cost-optimal plan $\plan^*$ such that $\mu(\pcs(\traj^{\plan^*},\Phi)) \leq \mu_{max}$, given a product automaton $\paut$. The pseudocode for constrained synthesis is presented in Alg.~\ref{alg:astar}.

Alg. \ref{alg:astar} takes in a product automaton $\paut$, an admissible heuristic function $h$, a preference function $\mu$, and a preference-cost constraint $\mu_{max}$. We introduce a method of generating a very efficient admissible heuristic $h$ in Sec. \ref{sec:heuristic}. The output of Alg. \ref{alg:astar} is a cost optimal plan $\plan^*$ that adheres to the preference-cost constraint. 

Similar to $BOA^*$ \cite{ulloa2020simple}, Alg. \ref{alg:astar} searches over ``nodes'', where each node captures a unique path to a given state $p$. 
We keep track of four properties of every node $x$: (i) $state(x)$, the corresponding product automaton state $p$, (ii) $g(x)$, the total cost or ``g-score'' of $x$, (iii) $\pcs(x)$, the PCS of the path leading to $x$, and (iv) $f(x) = g(x) + h(x)$, the heuristic ``f-score'' of $x$, which gives an under-estimate of the minimum cost-to-goal using the heuristic function $h$ (described in Sec. \ref{sec:heuristic}. Since $\mu$ is assumed to be monotonically-increasing along a trajectory, the preference cost of any partial-trajectory from $p_0$ to $state(x)$ is guaranteed to lower bound any trajectory through $x$. To guarantee that $\plan^*$ is optimal, we store the minimum cost encountered by any trajectory that reaches $p$, denoted $g_{min}(p)$. Additionally, to extract $\plan^*$, we store each node's parent inside $parent(x)$.

Initially, a new node $x_0$ is created and added to the open set $O$ with $state(x_0) = p_0$, quantities set to zero, and a $NULL$ parent (Lines 1-4). In each iteration, a node $x$ is selected from the open set $O$ with the lowest f-score (Line 6). If the product state $state(x)$ satisfies all tasks and attains the smallest g-score of any path to $x$, a solution is found, and the path can be extracted by recursively looking-up $parent(p)$ (Lines 7-8).
Otherwise, for each action enabling a transition on $\paut$ from state $state(x)$, it computes the tentative cost (represented by $cost(p, a) = c(s, a)$) and tentative PCS by adding the transition-PCS (Lines 9-11). If the preference-cost of the tentative trajectory exceeds the constraint, the tentative path is not allowed and can be disregarded (Line 12). Otherwise, the tentative path is a valid search candidate. If the tentative path yields a lower cost, the lower cost is marked using $g_{min}$ (Lines 14-15). Then, algorithm creates a new node $y$ representing the tentative path, assigns the tentative quantities, and marks the parent node $x$ (Lines 16-18). If the open set is ever empty, a plan that satisfies all task and adheres to the preference constraint does not exist, thus returning \textit{failure} (Line 19).

Alg. \ref{alg:astar} is guaranteed to return a cost-optimal plan $\plan^*$ that satisfies the preference-cost constraint $\mu(\pcs(\traj^{\plan^*})) \leq \mu_{max}$ if the heuristic function $h(p)$ is \textit{admissible} (see Sec. \ref{sec:heuristic}).

\subsection{Pareto Front Computation}

We are often interested in the Parent front, i.e., the
set of all optimal trade-off solutions. To efficiently  
compute such trade-offs
between cost and preference-cost, we formulate our problem as a multi-objective graph search problem, and leverage $BOA^*$ \cite{ulloa2020simple} to compute the Pareto front.

To use $BOA^*$, we first formulate our trade-off analysis synthesis problem as a simple bi-objective heuristic graph search problem where Objective 1 is cost and Objective 2 is preference-cost. A search problem can be defined using the tuple $(P, E, \mathbf{c}, p_0, \pacc, \mathbf{h})$ 
where $P$ is the set of states, $E \subseteq P \times P$ is a set of state transitions captured by the operator $neighbor$, $\mathbf{c} : E \mapsto \mathbb{R}^2_{\geq 0}$ is a cost-double edge weight function for two competing objectives, and $\textbf{h} : P \mapsto \mathbb{R}^2_{\geq 0}$ is a heuristic function (different from $h$ above). 
Similar to Alg. \ref{alg:astar}, Each node ($x$) in the search queue can be represented by the state $p$, cost-vector $\mathbf{g} = (g_1, g_2)$, and $f$-cost-double $\mathbf{f} = (f_1, f_2)$. Adapting to our problem, $g_1$ is the cost (denoted $g(x)$ in Alg. \ref{alg:astar}), and $g_2 = \mu(\pcs(x))$. For the purposes of this paper, we present a heuristic for only the cost objective, thus leaving $f_1 = f(x)$ (defined in Alg. \ref{alg:astar}) and $f_2 = g_2$.

\begin{figure*}[t]
    \centering
    \begin{subfigure}[t]{0.24\textwidth}
        \centering
        \includegraphics[trim={0 10mm 0 10mm},clip,width=1.1\linewidth]{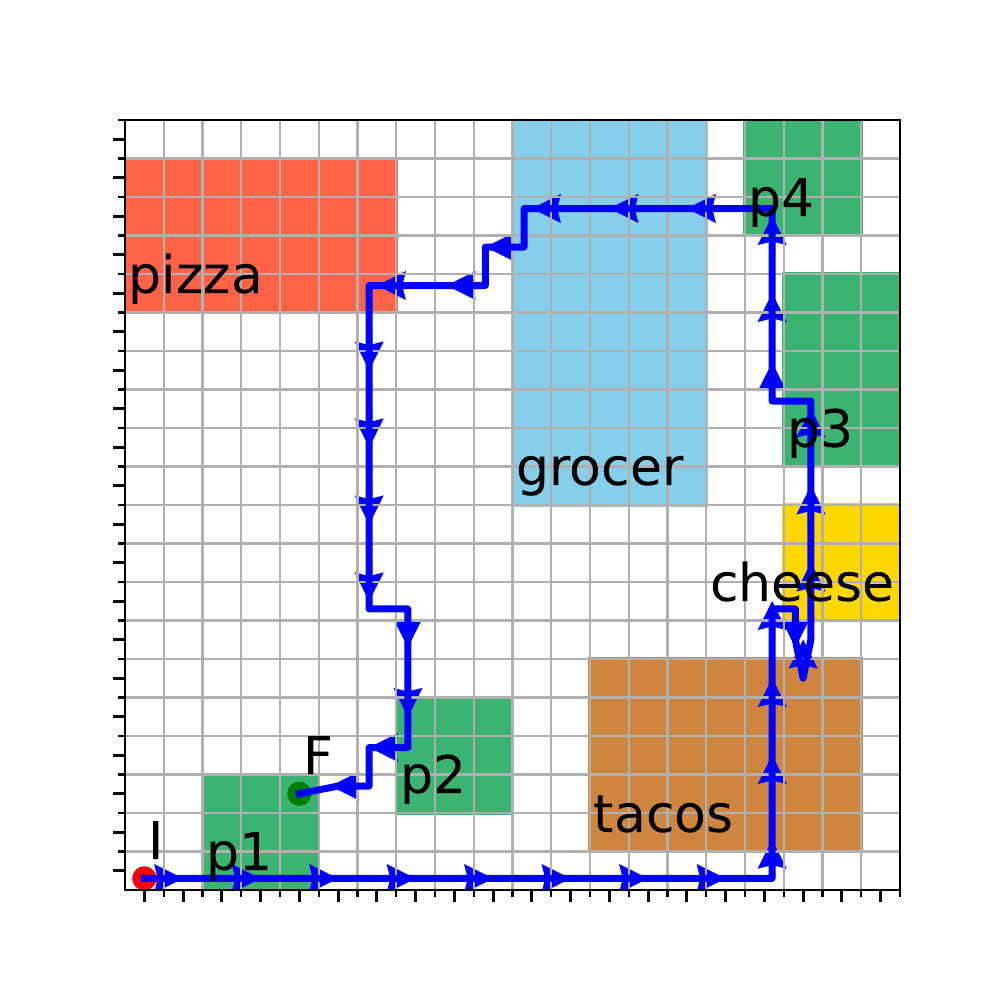}
        \caption{$\text{cost}=68$, $\mu=29$, }
    \label{fig: ac_1}
    \end{subfigure}%
    \begin{subfigure}[t]{0.24\textwidth}
        \centering
        \includegraphics[trim={0 10mm 0 10mm},clip,width=1.1\linewidth]{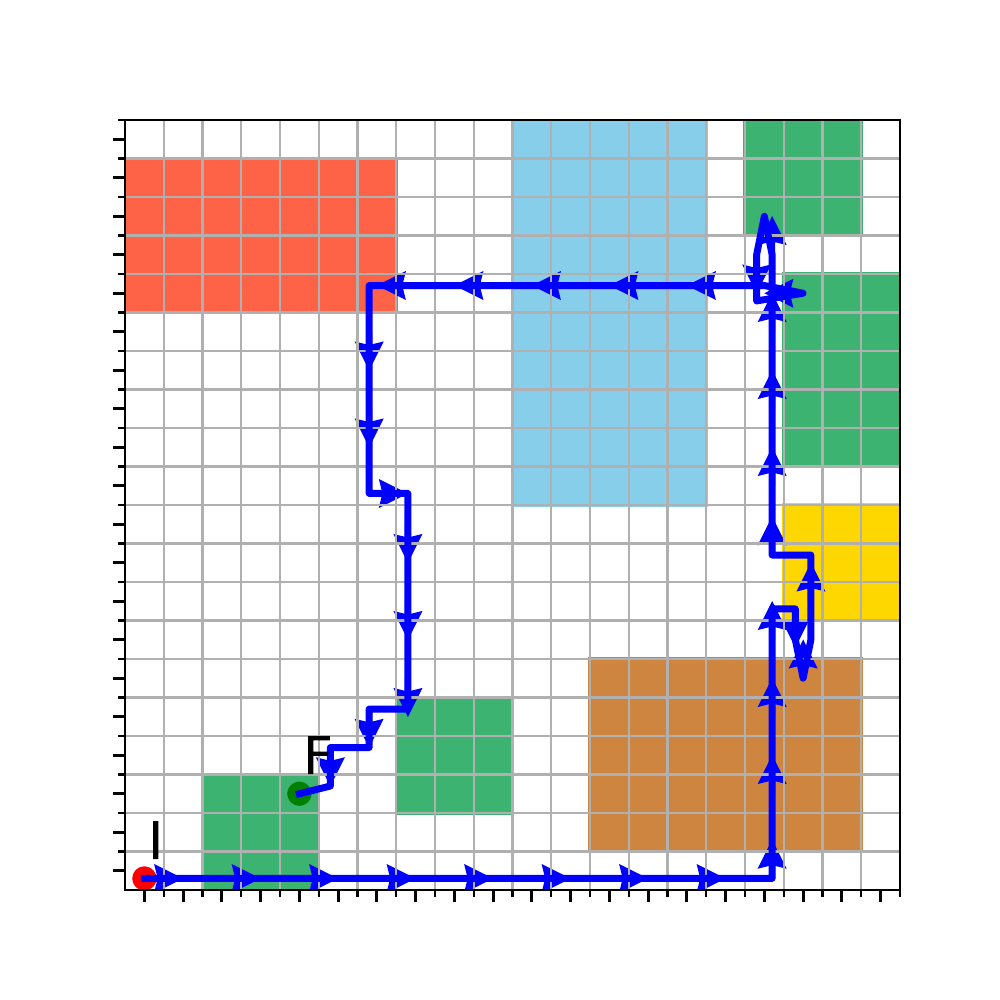}
        \caption{$\text{cost}=70$, $\mu=28$}
    \label{fig: ac_2}
    \end{subfigure}%
    \begin{subfigure}[t]{0.24\textwidth}
        \centering
        \includegraphics[trim={0 10mm 0 10mm},clip,width=1.1\linewidth]{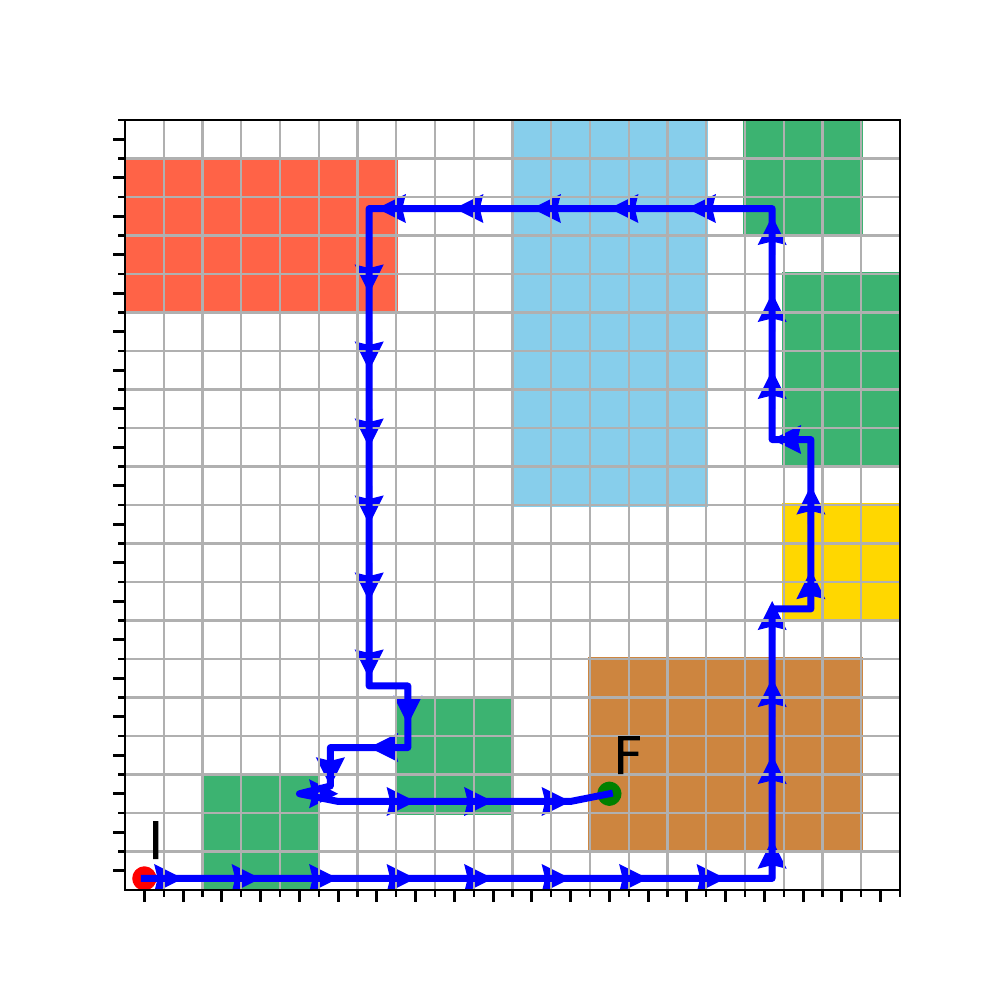}
        \caption{$\text{cost}=72$, $\mu=24$}
    \label{fig: ac_3}
    \end{subfigure}%
    \begin{subfigure}[t]{0.24\textwidth}
        \centering
        \includegraphics[trim={0 10mm 0 10mm},clip,width=1.1\linewidth]{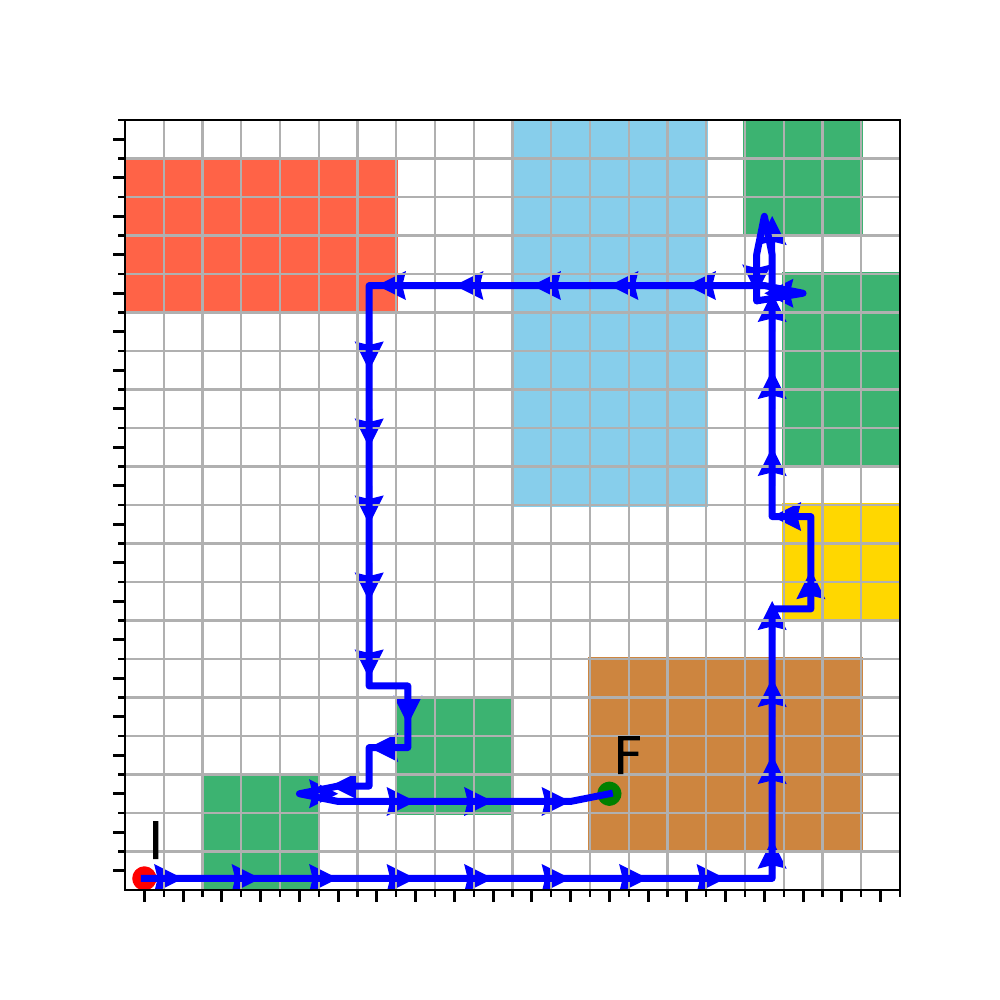}
        \caption{$\text{cost}=74$, $\mu=23$}
    \label{fig: ac_4}
    \end{subfigure}%
    \newline
    \begin{subfigure}[t]{0.24\textwidth}
        \centering
        \includegraphics[trim={0 10mm 0 10mm},clip,width=1.1\linewidth]{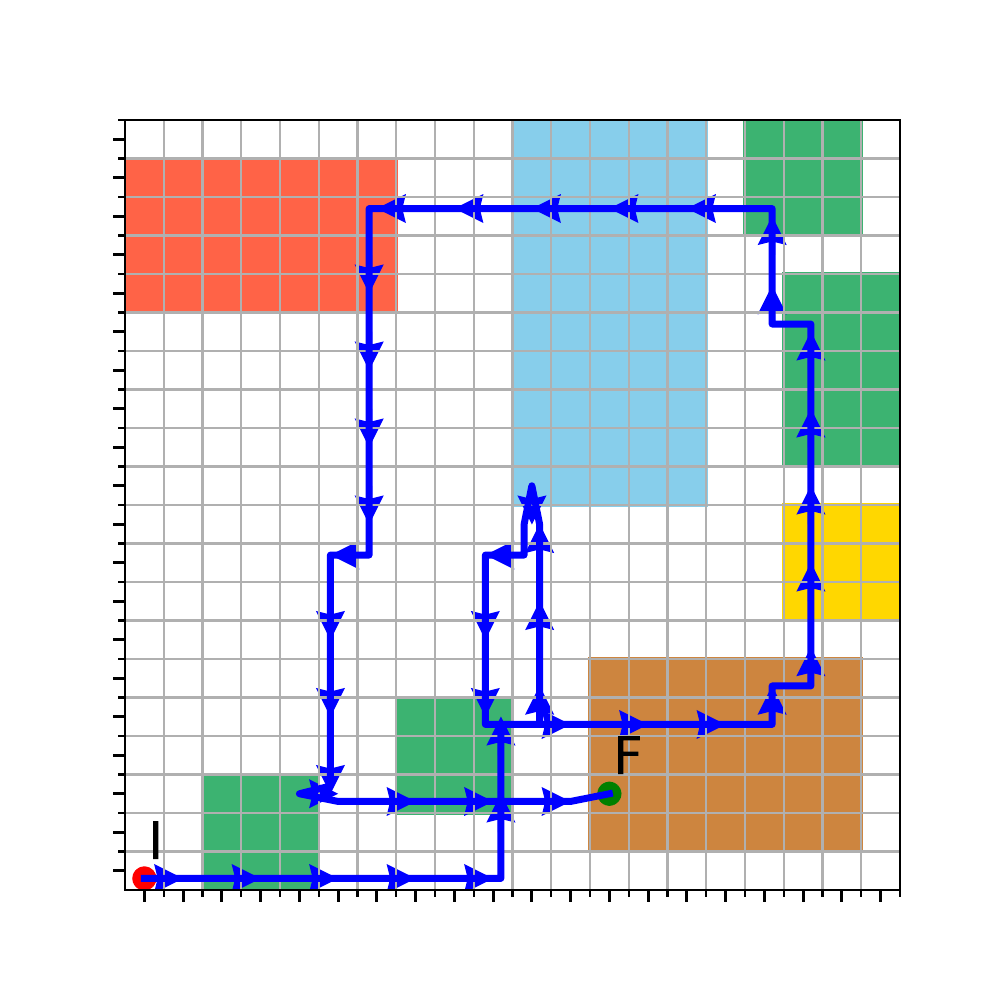}
        \caption{$\text{cost}=84$, $\mu=11$}
    \label{fig: ac_5}
    \end{subfigure}%
    \begin{subfigure}[t]{0.24\textwidth}
        \centering
        \includegraphics[trim={0 10mm 0 10mm},clip,width=1.1\linewidth]{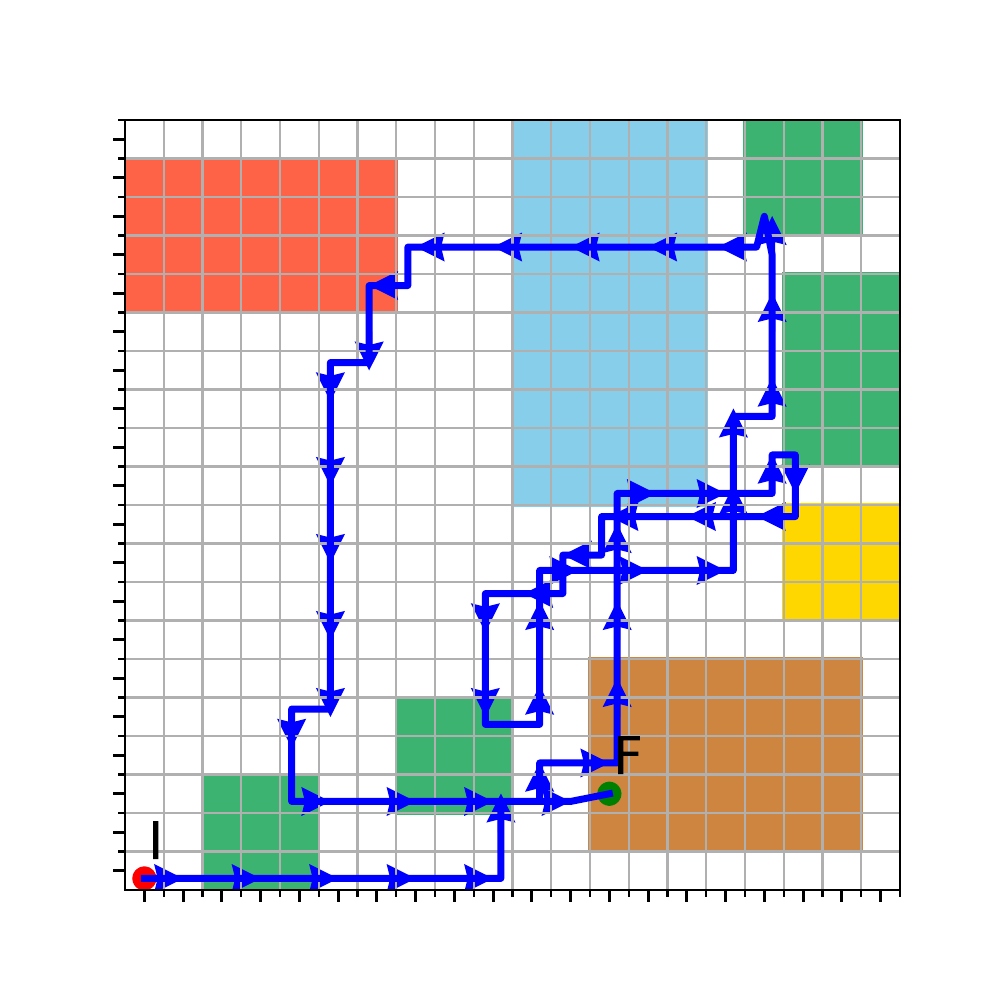}
        \caption{$\text{cost}=98$, $\mu=0$}
    \label{fig: ac_6}
    \end{subfigure}%
    \begin{subfigure}[t]{0.24\textwidth}
        \centering
        \includegraphics[trim={0 10mm 0 10mm},clip,width=1.1\linewidth]{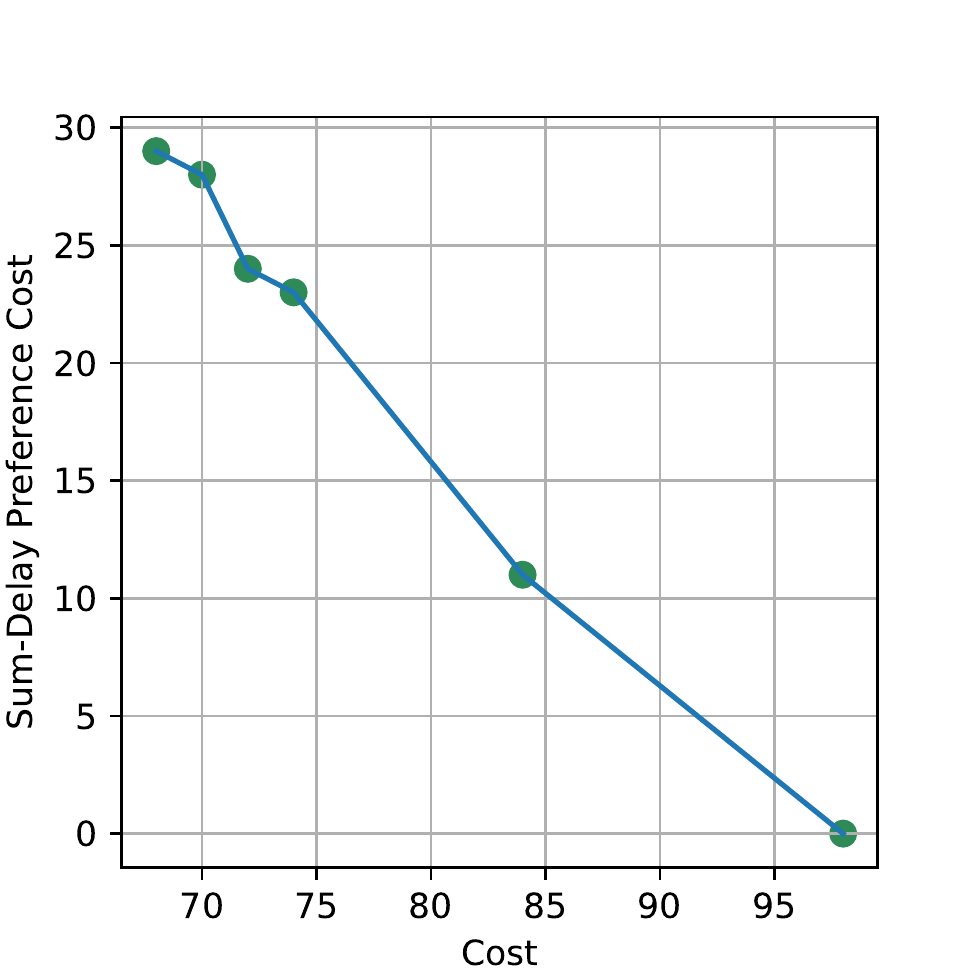}
        \caption{Pareto Front}
    \label{fig: ac_pf}
    \end{subfigure}%
    
\caption{Case Study 1 - Varying levels of order-of-satisfaction preference adhesion are shown for a mobile robot where (a) is the most cost-efficient plan, and (f) satisfies the tasks in the desired order. (g) is the Pareto front.}
\label{fig: ac}
\vspace{-4mm}
\end{figure*}

Tuple $(P, E, \mathbf{c}, p_0, \pacc, \mathbf{h})$ 
is used as an input to $BOA^*$ to compute all Pareto-optimal solutions with varying preference adherence.
Using an \textit{admissible} heuristic function, $BOA^*$ is guaranteed to find all optimal solutions if they exist.
\subsection{$A^*$ Symbolic Search}

Recall that for large formulas (complex task specifications), 
the size of each DFA $A_i$ might become very large, making the size of $\paut$ overwhelmingly large.
Hence, it is necessary to use a non-exhaustive search algorithm to find $\ptraj^*$.   
The first step towards increasing the computational efficiency is to represent $\paut$ implicitly.
This can done by embedding $\paut$ into the search algorithm and using the $neighbors$ operator.
All nodes recursively encountered using $neighbors$ can be saved, and the others are ignored.

Simply implicitly representing $\paut$ might improve the average-case total runtime and memory consumption; however, we can leverage the efficiency of the $A^*$ search algorithm to improve search speed and reduce the number of iterations needed to find the optimal plan. $A^*$ requires an \textit{admissible heuristic}, i.e., an underestimate of the cost
from a given node $p$ to the nearest $\pacc$,
to speed up the search while maintaining the optimality guarantee. Many applications of $A^*$ make use of external problem-specific information as a heuristic, such as Cartesian distance to goal on a map. To develop an approach that remains problem-agnostic, we use a novel method of pre-computing heuristic values using small-scale decentralized synthesis techniques. This approach makes no further assumptions on Problem
\ref{problem}, and provides significant runtime improvement (see benchmark results in Sec.~\ref{sec: results}).

\subsubsection{Decentralized Max-Min Heuristic Computation} \label{sec:heuristic}

To calculate the max-min heuristic value for a given node $p$, the minimum cost to satisfying each individual remaining task $\phi_j \in \Phi \setminus \acc{p}$
must be computed. Using the single product $\bar{\mathcal{P}}^j = T \otimes A_j$, the minimum cost-to-goal for a given node $\bar{p}^j\in \bar{P}^j$, denoted $d^j(\bar{p}^j)$ can be computed in a decentralized manner for all $\bar{p}^j \in \bar{P}^j$. The minimum cost-to-go is computed using Dijkstra's algorithm expanding backwards from all accepting states $\bar{p}^j_{acc}$, where $\bar{acc}^j(\bar{p}^j_{acc}) = \phi_j$ to all $\bar{p}^j \in \bar{P}^j$ yielding the minimum cost to any accepting state $d^j(\bar{p}^j)$.

Given a node $p = (s, q_1, ..., q_N)$ and a set of remaining formula indices $I = \{1\leq j \leq N \mid \phi_j \in \Phi \setminus \acc{p}\}$, the max-min heuristic $h : P \mapsto \mathbb{R}_{\geq 0}$ can be expressed as
\begin{equation}
    h(p) = \max\limits_{j\in I}(d^j(\bar{p}^j)) \text{ where }\; \bar{p}^j = (s, q^j).
\end{equation}

\begin{lemma}[Admissible Heuristic]
The Max-Min heuristic $h(p)$ is admissible, i.e., $h$ is an under-estimate of the minimum cost-to-go.
\end{lemma}
\begin{proof}
Let $\ptraj$ denote an arbitrary trajectory $\ptraj = \ptraj_0 , \ldots \ptraj_m$ with cost $\Cost{\ptraj}$ such that $\ptraj_0 = p$, $acc(\ptraj_m) = \Phi$ and $acc(\ptraj_{m-1}) \neq \Phi$. 
Assume there exists a $\ptraj^*$ such that $\Cost{\ptraj^*} < h(p)$.
Therefore, at least one remaining formula $\phi_j \in \Phi \setminus acc(p)$ has a minimum cost-to-go estimate $d^j(\bar{p}^j) > \Cost{\ptraj^*}$.
By the transition function in Def. \ref{def:paut}, the minimum cost between any two states in $\bar{\paut}^j$ must lower-bound the cost between equivalent states on $\paut$, hence a contradiction.

\end{proof}



Our experiments and benchmarks show that using this heuristic significantly improves the computation time.


\section{Experiments} 
\label{sec: results}


\begin{figure*}[t]
   \centering
   \begin{subfigure}[t]{0.24\textwidth}
       \centering
       \includegraphics[trim={0 10mm 0 10mm},clip,width=1.1\linewidth]{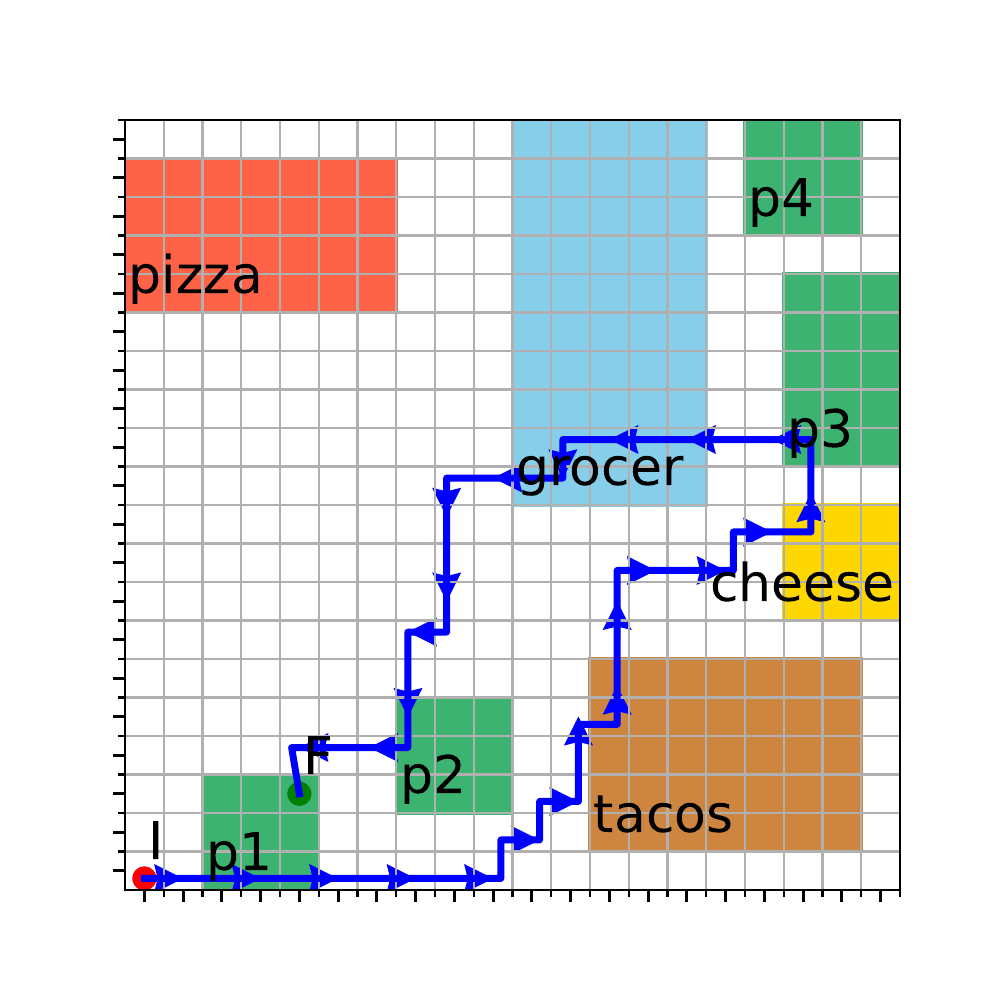}
       \caption{$\text{cost}=50$, $\mu=77$, }
   \label{fig: ps_1}
   \end{subfigure}%
   \begin{subfigure}[t]{0.24\textwidth}
       \centering
       \includegraphics[trim={0 10mm 0 10mm},clip,width=1.1\linewidth]{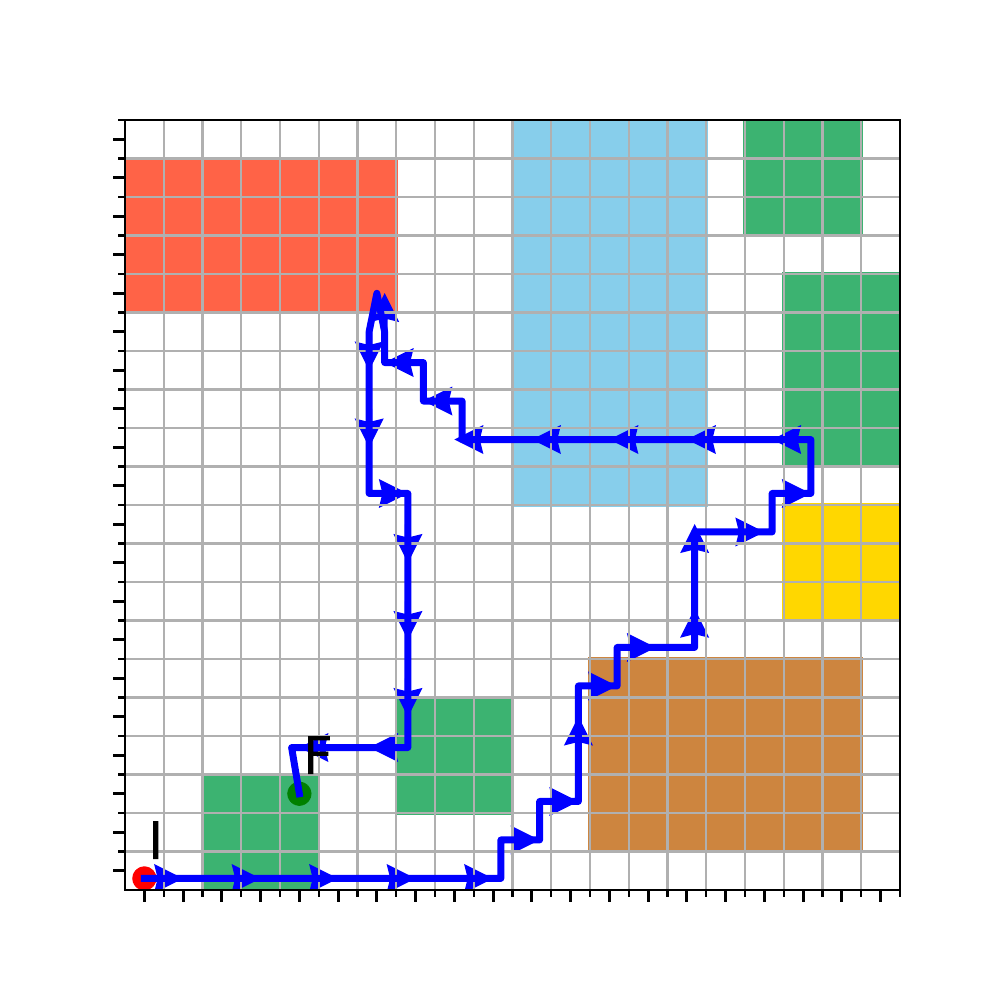}
       \caption{$\text{cost}=60$, $\mu=30$}
   \label{fig: ps_2}
   \end{subfigure}%
   \begin{subfigure}[t]{0.24\textwidth}
       \centering
       \includegraphics[trim={0 10mm 0 10mm},clip,width=1.1\linewidth]{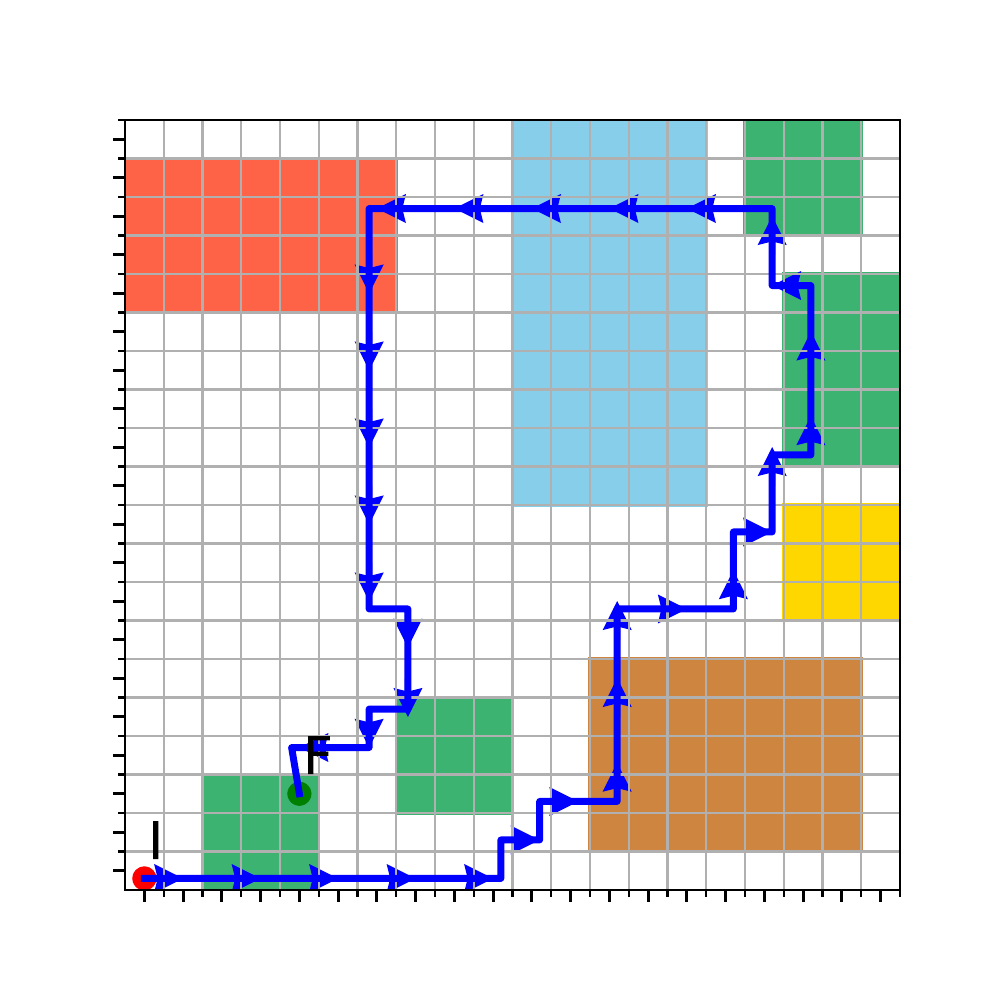}
       \caption{$\text{cost}=64$, $\mu=15$}
   \label{fig: ps_3}
   \end{subfigure}%
   ~~
   \begin{subfigure}[t]{0.24\textwidth}
       \centering
       \includegraphics[trim={0 10mm 0 10mm},clip,width=1.1\linewidth]{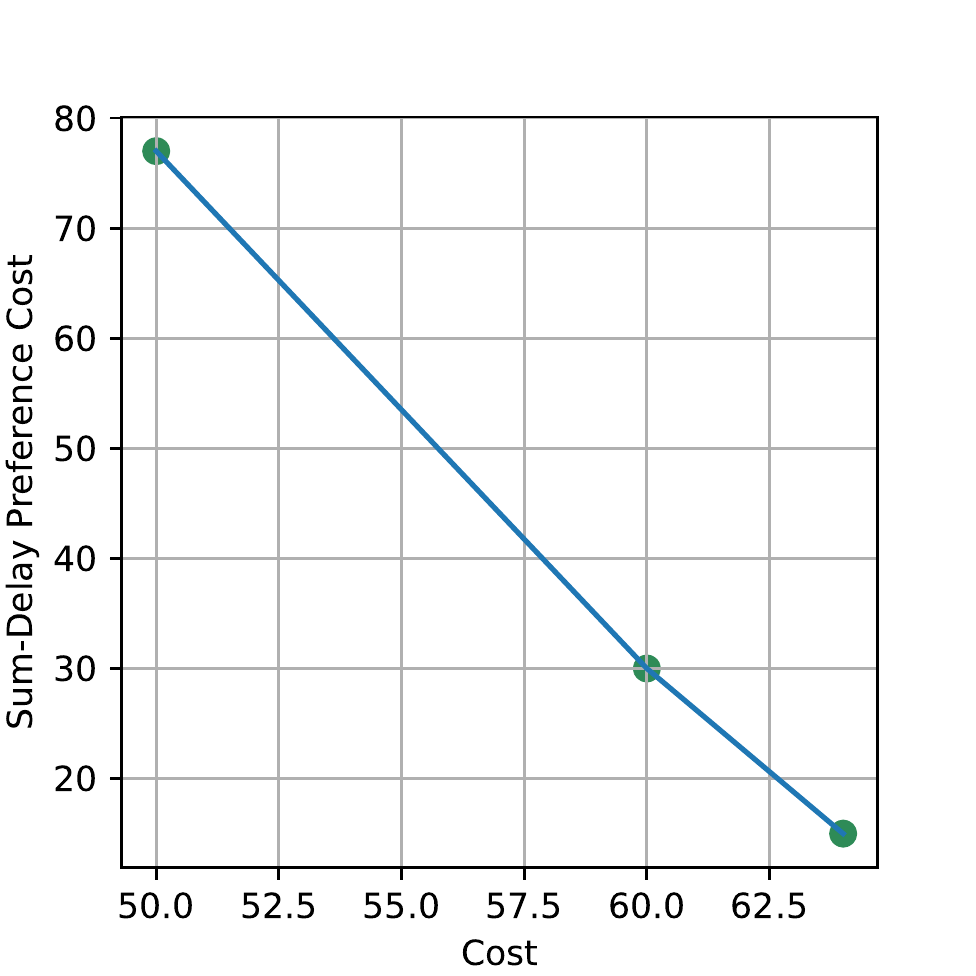}
       \caption{Pareto Front}
   \label{fig: ps_pf}
   \end{subfigure}%
\caption{Case Study 2 - Varying levels of preference adherence are shown for a mobile robot in a partial satisfaction scenario where (a) is the most cost-efficient plan and (c) minimizes the weighted-sum preference function over violation costs.}
\label{fig: ps}
\vspace{-2mm}
\end{figure*}

In this section, we demonstrate our novel preference formulation on realistic problems: a mobile robot and a manipulator. Additionally, since we aim to show the utility of our planning framework for more realistic/complex problems, we include benchmarks of both single-plan synthesis as well as Pareto front computation with and without the problem-agnostic heuristic described in Sec. \ref{sec:heuristic}. 



\begin{figure*}
    \centering
    \begin{subfigure}{0.43\textwidth}
        \centering
        \includegraphics[width=0.9\linewidth]{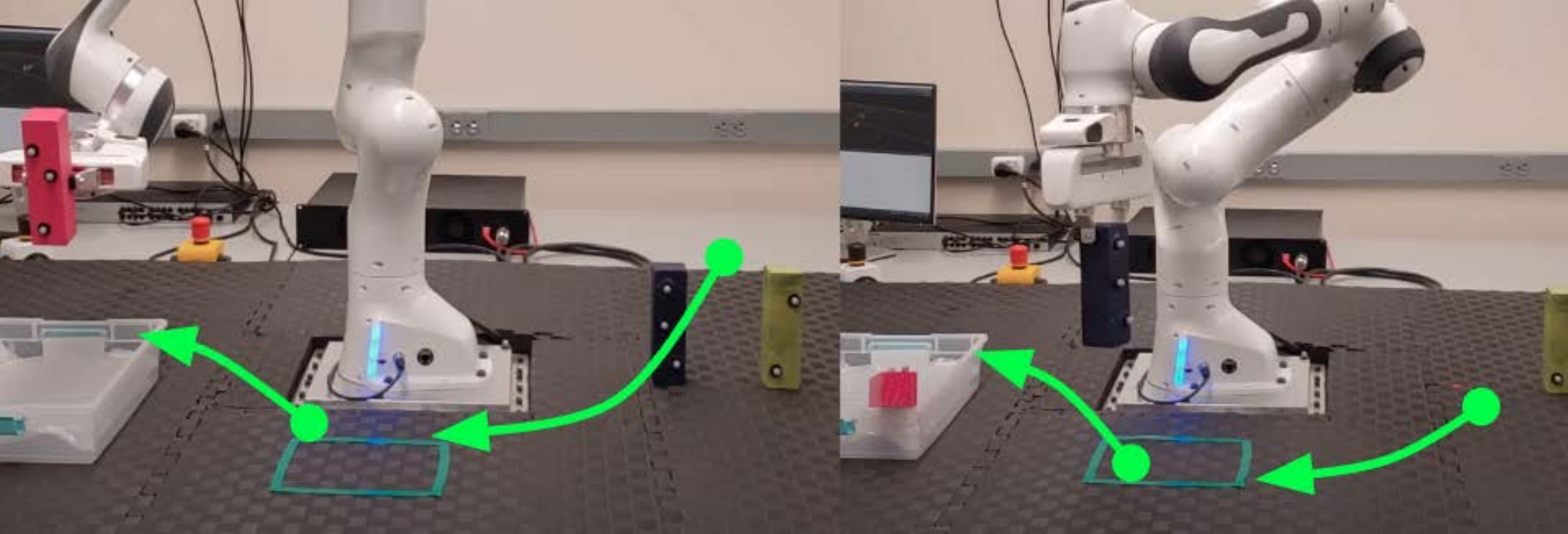}
        \caption{Cost-Optimal Plan}
    \label{fig: m_cost_op}
    \end{subfigure}%
    \begin{subfigure}{0.5750\textwidth}
        \centering
        \includegraphics[width=1.0\linewidth]{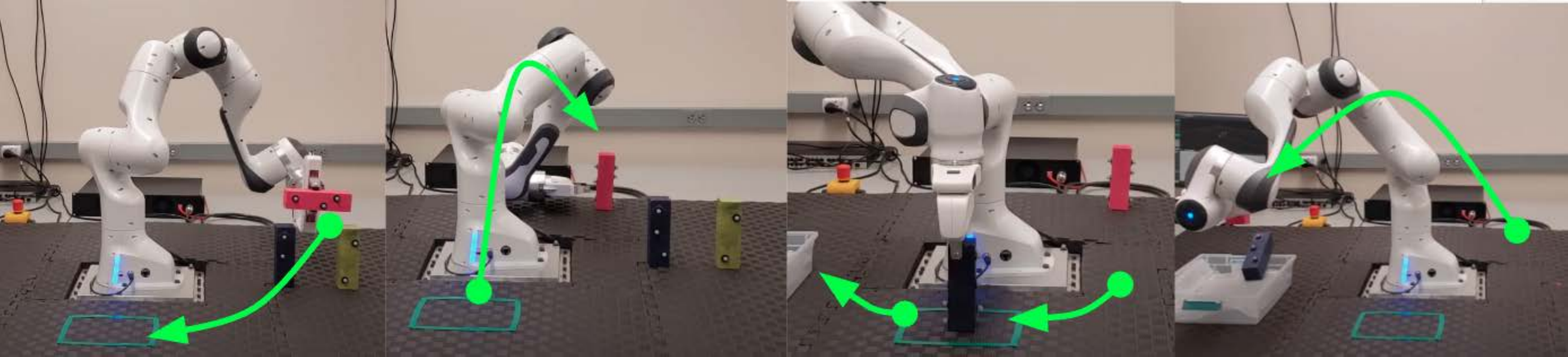}
        \caption{Preference-Optimal Plan}
    \label{fig: m_pref_opt}
    \end{subfigure}%
\caption{Case Study 3 - A cost-optimal manipulation plan is compared against the preference-optimal plan.}
\label{fig: manipulator}
\vspace{-4mm}
\end{figure*}

\subsubsection{Case Study 1: Order-of-Satisfaction Mobile-Robot}
Extending from Example~\ref{ex:wts}, we demonstrate the varying levels of preference adhesion using a 2D grid environment food-delivery robot simulation. For this case study, the environment is a $20\times 20$ grid-world equipped with cardinal-direction actions ($N$, $S$, $E$, $W$), each with a cost of $1$ unit. We give the robot five \csltl tasks: $\phi_1 = \neg pizza U cheese$ (\textit{do not pick up pizza until cheese factory is visited}), $\phi_2 = F grocer \land F(p2)$ (\textit{deliver groceries to person 2}), $\phi_3 = F (tacos \land F p3 \land F p4)$ (\textit{deliver tacos to person 3 and person 4}), $\phi_4 = F (pizza \land F p1)$ (\textit{deliver pizza to person 1}), and $\phi_5 = F(cheese \land F tacos \land F grocer \land F pizza)$ (\textit{deliver cheese to all food establishments}).
The layout of each location is shown in Fig. \ref{fig: ac_1}.
The preference function formulation in Example \ref{ex:pref_func} (preference over the order of satisfaction) is used. Fig.~\ref{fig: ac_1} shows the plan when only cost is optimized. Fig.~\ref{fig: ac_6} shows the resulting behavior when the user's preferred order is completely satisfied. Figs. \ref{fig: ac_2}-\ref{fig: ac_5} show varying levels of preference adherence. As the preference-cost increases, the robot finds ways to simultaneously work on tasks to save on cost.
Fig.~\ref{fig: ac_pf} shows the Pareto front.


\subsubsection{Case Study 2: Partial Satisfaction Mobile-Robot}
The environment from Case Study 1 is used. We task the robots with the same five formulae from Case Study 1, however we substitute the last formula for $\phi_5 = \neg cheese U pizza \land \neg cheese U tacos$ (\textit{do not visit cheese factory until you visit pizza and tacos}). With this change, $\phi_1$ has a conflicts with $\phi_5$. 
We adapt our algorithm as explained in Remarks 1-6 to perform partial-satisfaction planning with preference.

We introduce substitution costs for $\phi_1$: $15$ units to go to the grocer instead of the cheese factory, $\phi_3$: $15$ units to forego delivering food to person 4, $\phi_4$: $12$ units to pick up tacos instead of pizza, and $\phi_5$ can entirely be skipped for $20$ units. 
For this case study, we employ a simple weighted-sum of all the costs where the weight for $\phi_5$ is double the others.
Fig. \ref{fig: ps} shows all optimal trade-off plans with \ref{fig: ps_1} being the most cost-optimal and \ref{fig: ps_3} being the most preference optimal. Note that all plans incur a non-zero preference cost since there is no feasible plan that can fully satisfy all tasks. However, as seen in \ref{fig: ps_1} and \ref{fig: ps_2}, our planner is able to compute optimal partially satisfying solutions such that the robot skips parts of certain tasks (such as delivering food to person 4 and visiting pizza) to save on cost. The Pareto front is shown in \ref{fig: ps_pf}. 




\subsubsection{Mobile-Robot Benchmarks}
We benchmarked the two planning algorithms with and without the heuristic on a $10\times 10$ grid-world with 5-state DFAs representing formulae of the form $F(s_a \land F(s_b) \land F(s_c))$ where $s_a$, $s_b$, and $s_c$ are observations of three random cells. One hundred randomized trials were performed.
Table \ref{tab:bm} shows the results.




\begin{table}[t]
    \centering
    \caption{\small Benchmark results for 10$\times$10 Grid-Robot over 100 runs.}
    \begin{tabular}{ c|l|l|l|l } 
    \toprule
        \multirow{1}{*}{} & \multicolumn{2}{c|}{Single Plan Runtime (s)} & \multicolumn{2}{c}{Pareto Front Runtime (s)} \\ 
          N & \multicolumn{1}{c|}{w/o $h$} & \multicolumn{1}{c|}{w/ $h$} & \multicolumn{1}{c|}{w/o $h$} & \multicolumn{1}{c}{w/ $h$} \\ 
         \hline 
        2 & $3.88 \times 10^{-2}$ & $7.50\times 10^{-3}$ & $4.08\times 10^{-2}$ & $9.00\times 10^{-3}$ \\ 
         3 & $2.48\times 10^{-1}$ & $2.87\times 10^{-2}$ & $2.69 \times 10^{-1}$ & $5.26\times 10^{-2}$ \\
         4 & $1.49\times 10^{0}$ & $1.05\times 10^{-1}$ & $1.67 \times 10^{0}$ & $3.10 \times 10^{-1}$ \\ 
         5 & $8.99\times 10^{0}$ & $4.75\times 10^{-1}$ & $1.02\times 10^{1}$ & $1.98\times 10^{0}$ \\ 
         6 & $4.85\times 10^{1}$ & $1.98\times 10^{0}$ & $5.93\times 10^{1}$ & $1.13\times 10^{1}$ \\ 
         7 & $2.50\times 10^{2}$ & $\mathbf{7.50\times 10^{0}}$ & $3.34\times 10^{2}$ & $6.26\times 10^{1}$ \\ 
         8 & $1.39\times 10^{3}$ & $\mathbf{3.65\times 10^{1}}$ & $1.86\times 10^{3}$ & $3.48\times 10^{2}$ \\ 
        \bottomrule
     \end{tabular}
     \label{tab:bm}
\end{table}


We observe three important points: (i) the runtime increases exponentially with the number of formulae, (ii) the proposed heuristic reduces computation time by at least one order of magnitude on both algorithms and in some cases by two orders of magnitude (shown in bold font) , and (iii) as $N$ increases the computational efficacy of the heuristic becomes more dominant. 

\subsubsection{Manipulation Case Study}
To show the generality of the approach, we performed experiments on the Franka Emika ``Panda'' 7-DOF robotic manipulator. 
We used our algorithm as the high-level planner in the framework introduced in \cite{He:ICRA:2015, He:IROS:2017, Muvvala:2022:ICRA}. The model captures 4 types of actions $A=\{Grasp, Release, Transit, Transport\}$ with respective costs of $1$, $1$, $3$, and $5$ units. 




In the scenario shown in Fig. \ref{fig: manipulator}, there is a blue block $b^b$ and a pink block $b^p$ that need to be sanitized and binned. The model for this system is complex with three objects 
 and seven locations, resulting in a transition system with over $6$K states. The blocks are initially arranged in an arch with $b^p$ on top, $b^b$ in the left base, and a green block in the right base. We assume $b^g$ is stationary for this demonstration. The robot is tasked with sanitizing and binning $b^b$, then sanitizing and binning $b^p$. However, the tasks capture that a block cannot be on top unless there are two base blocks for support, i.e., $\phi_1 = (\neg b^b_{top}) \land (b^p_{top} \implies b^b_{left}) U (b^b_{san} \land b^b_{bin1})$. Formula $\phi_2$
is similar to $\phi_1$ except $b^b$ is replaced with $b^p$ on the right-hand-side of $U$. As shown in figure \ref{fig: m_cost_op}, the cost optimal plan computed using our framework sanitizes and bins the pink block, then does the same for the blue block. In figure \ref{fig: m_pref_opt}, the preference optimal plan must move the pink block off of the top to sanitize and bin the blue block first before binning the pink block.
The videos of the robot executions for each plan are provided in \cite{videos}.


\section{Conclusion}
\label{sec: conclusion}

In this work, 
we consider optimal preference-cost trade-off planning for robotic systems with multiple tasks. 
We introduce a novel formulation for the user's preference that reasons over the cost of satisfying each individual task. 
Further, we present algorithmic adaptations of the $A^*$ search algorithm
with a novel heuristic
to efficiently synthesize plans with a desired user preference, as well as Pareto front analysis between \textit{cost} and \textit{preference}.  Our method is general with applications in both mobile robotics and manipulation.  Our benchmarks show our heuristic significantly reduces 
computation time.  


\bibliographystyle{IEEEtran}
\bibliography{games_references,lahijanian,amorese}

\end{document}